\documentclass[sigconf]{acmart}

\usepackage{amsthm}
\usepackage{amsfonts}
\usepackage{color}
\usepackage{xcolor}
\usepackage{multirow}
\usepackage{tabularx}
\usepackage{subcaption}
\usepackage{graphicx}
\usepackage[ruled]{algorithm2e}

\AtBeginDocument{%
  \providecommand\BibTeX{{%
    \normalfont B\kern-0.5em{\scshape i\kern-0.25em b}\kern-0.8em\TeX}}}

\copyrightyear{2022}
\acmYear{2022}
\setcopyright{acmcopyright}
\acmConference[CIKM '22] {Proceedings of the 31st ACM International Conference on Information and Knowledge Management}{October 17--21, 2022}{Atlanta, GA, USA.}
\acmBooktitle{Proceedings of the 31st ACM International Conference on Information and Knowledge Management (CIKM '22), October 17--21, 2022, Atlanta, GA, USA}
\acmPrice{15.00}
\acmISBN{978-1-4503-9236-5/22/10}
\acmDOI{10.1145/3511808.3557490}

\begin{document}

\title{X-GOAL: Multiplex Heterogeneous Graph Prototypical Contrastive Learning}


\author{Baoyu Jing}
\email{baoyuj2@illinois.edu}
\affiliation{%
  \institution{University of Illinois at Urbana-Champaign}
    \country{}
}

\author{Shengyu Feng}
\email{shengyuf@andrew.cmu.edu}
\affiliation{%
  \institution{Language Technology Institute\\Carnegie Mellon University}
    \country{}
}

\author{Yuejia Xiang}
\email{yuejiaxiang@tencent.com}
\affiliation{%
  \institution{Platform and Content Group\\Tencent}
    \country{}
}

\author{Xi Chen}
\email{jasonxchen@tencent.com}
\affiliation{%
  \institution{Platform and Content Group\\Tencent}
    \country{}
}

\author{Yu Chen}
\email{andyyuchen@tencent.com}
\affiliation{%
  \institution{Platform and Content Group\\Tencent}
  \country{}
}

\author{Hanghang Tong}
\email{htong@illinois.edu}
\affiliation{%
  \institution{University of Illinois at Urbana-Champaign}
  \country{}
}

\renewcommand{\shortauthors}{Baoyu Jing et al.}

\begin{abstract}
Graphs are powerful representations for relations among objects, which have attracted plenty of attention in both academia and industry.
A fundamental challenge for graph learning is how to train an effective Graph Neural Network (GNN) encoder without labels, which are expensive and time consuming to obtain.
Contrastive Learning (CL) is one of the most popular paradigms to address this challenge, which trains GNNs by discriminating positive and negative node pairs.
Despite the success of recent CL methods, there are still two under-explored problems.
Firstly, how to reduce the semantic error introduced by random topology based data augmentations.
Traditional CL defines positive and negative node pairs via the node-level topological proximity, which is solely based on the graph topology regardless of the semantic information of node attributes, and thus some semantically similar nodes could be wrongly treated as negative pairs.
Secondly, how to effectively model the multiplexity of the real-world graphs, where nodes are connected by various relations and each relation could form a homogeneous graph layer.
To solve these problems, we propose a novel multiple\underline{x} heterogeneous \underline{g}raph pr\underline{o}totypical contr\underline{a}stive \underline{l}eaning (X-GOAL) framework to extract node embeddings.
X-GOAL is comprised of two components: the GOAL framework, which learns node embeddings for each homogeneous graph layer, and an alignment regularization, which jointly models different layers by aligning layer-specific node embeddings.
Specifically, the GOAL framework captures the node-level information by a succinct graph transformation technique, and captures the cluster-level information by pulling nodes within the same semantic cluster closer in the embedding space.
The alignment regularization aligns embeddings across layers at both node level and cluster level. 
We evaluate the proposed X-GOAL on a variety of real-world datasets and downstream tasks to demonstrate the effectiveness of the X-GOAL framework.

\end{abstract}



\maketitle

\section{Introduction}
Graphs are powerful representations of formalisms and have been widely used to model relations among various objects \cite{hamilton2017representation, kipf2016semi, tang2015line, yan2021dynamic, zhou2019misc, zhou2020data, yan2021bright}, such as the citation relation and the same-author relation among papers.
One of the primary challenges for graph representation learning is how to effectively encode nodes into informative embeddings such that they can be easily used in downstream tasks for extracting useful knowledge \cite{hamilton2017representation}.
Traditional methods, such as Graph Convolutional Network (GCN) \cite{kipf2016semi}, leverage human labels to train the graph encoders. 
However, human labeling is usually time-consuming and expensive, and the labels might be unavailable in practice \cite{wu2021self, liu2021graph, zheng2021heterogeneous, zheng2021tackling, du2021hypergraph}.
Self-supervised learning \cite{wu2021self, liu2021graph}, which aims to train graph encoders without external labels, has thus attracted plenty of attention in both academia and industry.

One of the predominant self-supervised learning paradigms in recent years is Contrastive Learning (CL), which aims to learn an effective Graph Neural Network (GNN) encoder such that positive node pairs will be pulled together and negative node pairs will be pushed apart in the embedding space \cite{wu2021self}.
Early methods, such as DeepWalk \cite{perozzi2014deepwalk} and node2vec \cite{grover2016node2vec}, sample positive node pairs based on their local proximity in graphs.
Recent methods rely on graph transformation or augmentation \cite{wu2021self} to generate positive pairs and negative pairs, such as random permutation \cite{velivckovic2018deep, hu2019strategies, jing2021hdmi}, structure based augmentation \cite{hassani2020contrastive, you2020graph}, sampling based augmentation \cite{qiu2020gcc, 9338425} as well as adaptive augmentation \cite{zhu2021graph}.

Albeit the success of these methods, they define positive and negative node pairs based upon the node-level information (or local topological proximity) but have not fully explored the cluster-level (or semantic cluster/prototype) information.
For example, in an academic graph, two papers about different sub-areas in graph learning (e.g., social network analysis and drug discovery) might not topologically close to each other since they do not have a direct citation relation or same-author relation.
Without considering their semantic information such as the keywords and topics, these two papers could be treated as a negative pair by most of the existing methods.
Such a practice will inevitably induce semantic errors to node embeddings, which will have a negative impact on the performance of machine learning models on downstream tasks such as classification and clustering.
To address this problem, inspired by \cite{li2020prototypical}, we introduce a \underline{g}raph pr\underline{o}totypical contr\underline{a}stive \underline{l}earning (GOAL) framework to simultaneously capture both node-level and cluster-level information.
At the node level, GOAL trains an encoder by distinguishing positive and negative node pairs, which are sampled by a succinct graph transformation technique.
At the cluster level, GOAL employs a clustering algorithm to obtain the semantic clusters/prototypes and it pulls nodes within the same cluster closer to each other in the embedding space.

Furthermore, most of the aforementioned methods ignore the multiplexity \cite{park2020unsupervised,jing2021hdmi} of the real-world graphs, where nodes are connected by multiple types of relations and each relation formulates a layer of the multiplex heterogeneous graph.
For example, in an academic graph, papers are connected via the same authors or the citation relation; 
in an entertainment graph, movies are linked through the shared directors or actors/actresses; 
in a product graph, items have relations such as also-bought and also-view.
Different layers could convey different and complementary information.
Thus jointly considering them could produce more informative embeddings than separately treating different layers and then applying average pooling over them to obtain the final embeddings \cite{jing2021hdmi, park2020unsupervised}.
Most of the prior deep learning methods use attention mechanism \cite{park2020deep, jing2021hdmi, wang2019heterogeneous, cen2019representation, ma2018multi, ma2019multi} to combine embeddings from different layers.
However, attention modules usually require extra tasks or loss functions to train, such as node classification \cite{wang2019heterogeneous} and concensus loss \cite{park2020deep}.
Besides, some attention modules are complex which require significant amount of extra efforts to design and tune, such as the hierarchical structures \cite{wang2019heterogeneous} and complex within-layer and cross-layer interactions \cite{ma2019multi}.
Different from the prior methods, we propose an alternative nimble alignment regularization to jointly model and propagate information across different layers by aligning the layer-specific embeddings without extra neural network modules, and the final node embeddings are obtained by simply average pooling over these layer-specific embeddings.
The key assumption of the alignment regularization is that layer-specific embeddings of the same node should be close to each other in the embedding space and they should also be semantically similar.
We also theoretically prove that the proposed alignment regularization could effectively maximize the mutual information across layers.

We comprehensively evaluate X-GOAL on a variety of real-world attributed multiplex heterogeneous graphs. 
The experimental results show that the embeddings learned by GOAL and X-GOAL could outperform state-of-the-art methods of homogeneous graphs and multiplex heterogeneous graphs on various downstream tasks.

The main contributions are summarized as follows: 
\begin{itemize}
    \item \textbf{Method.} We propose a novel X-GOAL framework to learn node embeddings for multiplex heterogeneous graphs, which is comprised of a GOAL framework for each single layer and an alignment regularization to propagate information across different layers. GOAL reduces semantic errors, and the alignment regularization is nimbler than attention modules for combining layer-specific node embeddings.
    \item \textbf{Theoretical Analysis.} We theoretically prove that the proposed alignment regularization can effectively maximize the mutual information across layers.
    \item \textbf{Empirical Evaluation.} We comprehensively evaluate the proposed methods on various real-world datasets and downstream tasks. The experimental results show that GOAL and X-GOAL outperform the state-of-the-art methods for homogeneous and multiplex heterogeneous graphs respectively.
\end{itemize}


\section{Preliminary}\label{sec:preliminary}


\begin{figure}[t!]
    \centering
    \includegraphics[width=.35\textwidth]{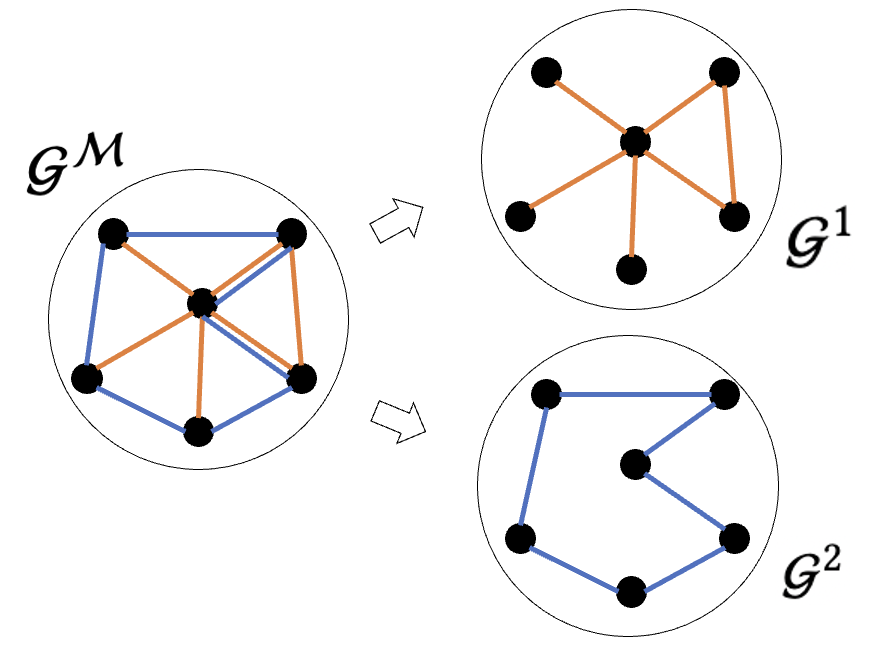}
    \caption{Illustration of the multiplex heterogeneous graph $\mathcal{G^M}$, which can be decomposed into homogeneous graph layers $\mathcal{G}^1$ and $\mathcal{G}^2$ according to the types of relations. Different colors represent different relations.}
    \label{fig:multiplex}
\end{figure}

\begin{definition}[Attributed Multiplex Heterogeneous Graph]
An attributed multiplex heterogeneous graph with $V$ layers and $N$ nodes is denoted as
$\mathcal{G^M}=\{\mathcal{G}^v\}_{v=1}^V$, where $\mathcal{G}^v{(\mathbf{A}^v, \mathbf{X})}$ is the $v$-th homogeneous graph layer,  $\mathbf{A}^v\in\mathbb{R}^{N\times N}$ and $\mathbf{X}\in\mathbb{R}^{N\times d_x}$ is the adjacency matrix and the attribute matrix, and $d_x$ is the dimension of attributes.
An illustration is shown in Figure \ref{fig:multiplex}.
\end{definition}

\noindent\textbf{Problem Statement.}
The task is to learn an encoder $\mathcal{E}$ for $\mathcal{G^M}$, which maps the node attribute matrix $\mathbf{X}\in\mathbb{R}^{N\times d_x}$ to node embedding matrix $\mathbf{H}^{\mathcal{M}}\in\mathbb{R}^{N\times d}$ without external labels, where $N$ is the number of nodes, $d_x$ and $d$ are the dimension sizes. 

\section{Methodology}\label{sec:method}
We present the X-GOAL framework for multiplex heterogeneous graphs $\mathcal{G^M}$, which is comprised of a GOAL framework and an alignment regularization. 
In Section \ref{subsec:pcl}, we present the GOAL framework, which simultaneously captures the node-level and the cluster-level information for each layer $\mathcal{G}=(\mathbf{A}, \mathbf{X})$ of $\mathcal{G^M}$.
In Section \ref{subsec:mvp}, we introduce a novel alignment regularization to align node embeddings across layers at both node and cluster level.
In section \ref{subsec:theory}, we provide theoretical analysis of the alignment regularization.

\subsection{The GOAL Framework}\label{subsec:pcl}
The node-level graph topology based transformation techniques might contain semantic errors since they ignore the hidden semantics and will inevitably pair two semantically similar but topologically far nodes as a negative pair.
To solve this issue, we introduce a GOAL framework for each homogeneous graph layer\footnote{For clarity, we drop the script $v$ of $\mathcal{G}^v$, $\mathbf{A}^v$ and $\mathbf{H}^v$ for this subsection.} $\mathcal{G}=(\mathbf{A}, \mathbf{X})$ to capture both node-level and cluster-level information.
An illustration of GOAL is shown in Figure \ref{fig:pcl}. 
Given a homogeneous graph $\mathcal{G}$ and an encoder $\mathcal{E}$, GOAL alternatively performs semantic clustering and parameter updating.
In the semantic clustering step, a clustering algorithm $\mathcal{C}$ is applied over the embeddings $\mathbf{H}$ to obtain the hidden semantic clusters.
In the parameter updating step, GOAL updates the parameters of $\mathcal{E}$ by the loss $\mathcal{L}$ given in Equation \eqref{eq:pcl}, which pulls topologically similar nodes closer and nodes within the same semantic cluster closer by the node-level loss and the cluster-level loss respectively. 

\noindent\textbf{A - Node-Level Loss.}\label{subsec:pcl_node}
To capture the node-level information, we propose a graph transformation technique $\mathcal{T} = \{\mathcal{T}^+, \mathcal{T}^-\}$, where $\mathcal{T}^+$ and $\mathcal{T}^-$ denote positive and negative transformations, along with a contrastive loss similar to InfoNCE \cite{oord2018representation}.

Given an original homogeneous graph $\mathcal{G}=(\mathbf{A}, \mathbf{X})$, the positive transformation $\mathcal{T}^+$ applies the dropout operation \cite{srivastava2014dropout} over $\mathbf{A}$ and $\mathbf{X}$ with a pre-defined probability $p_{drop}\in(0, 1)$.
We choose the dropout operation rather than the masking operation since the dropout re-scales the outputs by $\frac{1}{1-p_{drop}}$ during training, which improves the training results.
The negative transformation $\mathcal{T}^-$ is the random shuffle of the rows for $\mathbf{X}$ \cite{velivckovic2018deep}.
The transformed positive and negative graphs are denoted by $\mathcal{G}^{+}=\mathcal{T}^+(\mathcal{G})$ and $\mathcal{G}^{-}=\mathcal{T}^-(\mathcal{G})$, respectively.
The node embedding matrices of $\mathcal{G}$, $\mathcal{G}^+$ and $\mathcal{G}^-$ are thus $\mathbf{H}=\mathcal{E}(\mathcal{G})$, $\mathbf{H}^+=\mathcal{E}(\mathcal{G}^+)$ and $\mathbf{H}^-=\mathcal{E}(\mathcal{G}^-)$.

We define the node-level contrastive loss as:
\begin{equation}\label{eq:node_loss}
    \mathcal{L}_\mathcal{N} = -\frac{1}{N}\sum_{n=1}^N\log\frac{\mathbf{e}^{cos(\mathbf{h}_n, \mathbf{h}_n^+)}}{\mathbf{e}^{cos(\mathbf{h}_n, \mathbf{h}_n^+)} + \mathbf{e}^{cos(\mathbf{h}_n, \mathbf{h}_n^-)}}
\end{equation}
where $cos(,)$ denotes the cosine similarity, $\mathbf{h}_n$, $\mathbf{h}_n^+$ and $\mathbf{h}_n^-$ are the $n$-th rows of $\mathbf{H}$, $\mathbf{H}^+$ and $\mathbf{H}^-$.

\noindent\textbf{B - Cluster-Level Loss.}\label{subsec:pcl_semantic}
We use a clustering algorithm $\mathcal{C}$ to obtain the semantic clusters of nodes $\{\mathbf{c}_k\}_{k=1}^K$, where $\mathbf{c}_k\in\mathbb{R}^{d}$ is the cluster center, $K$ and $d$ are the number of clusters and the dimension of embedding space.
We capture the cluster-level semantic information to reduce the semantic errors by pulling nodes within the same cluster closer to their assigned cluster center.
For clarity, the derivations of the cluster-level loss are provided in Appendix.

We define the probability of $\mathbf{h}_n$ belongs to the cluster $k$ by: 
\begin{equation}\label{eq:prob}
    p(k|\mathbf{h}_n) = \frac{\mathbf{e}^{(\mathbf{c}_{k}^T\cdot\mathbf{h}_n/\tau)}}{\sum_{k'=1}^K\mathbf{e}^{(\mathbf{c}_{k'}^T\cdot\mathbf{h}_n/\tau)}}
\end{equation}
where $\tau >0$ is the temperature parameter to re-scale the values.

The cluster-level loss is defined as the negative log-likelihood of the assigned cluster $k_n$ for $\mathbf{h}_n$:
\begin{equation}\label{eq:semantic_loss}
    \mathcal{L}_\mathcal{C} = -\frac{1}{N}\sum_{n=1}^N\log \frac{\mathbf{e}^{(\mathbf{c}_{k_n}^T\cdot\mathbf{h}_n/\tau)}}{\sum_{k=1}^K\mathbf{e}^{(\mathbf{c}_{k}^T\cdot\mathbf{h}_n/\tau)}}
\end{equation}
where $k_n\in[1, \dots, K]$ is the cluster index assigned to the $n$-th node.

\begin{figure}
    \centering
    \includegraphics[width=0.4\textwidth]{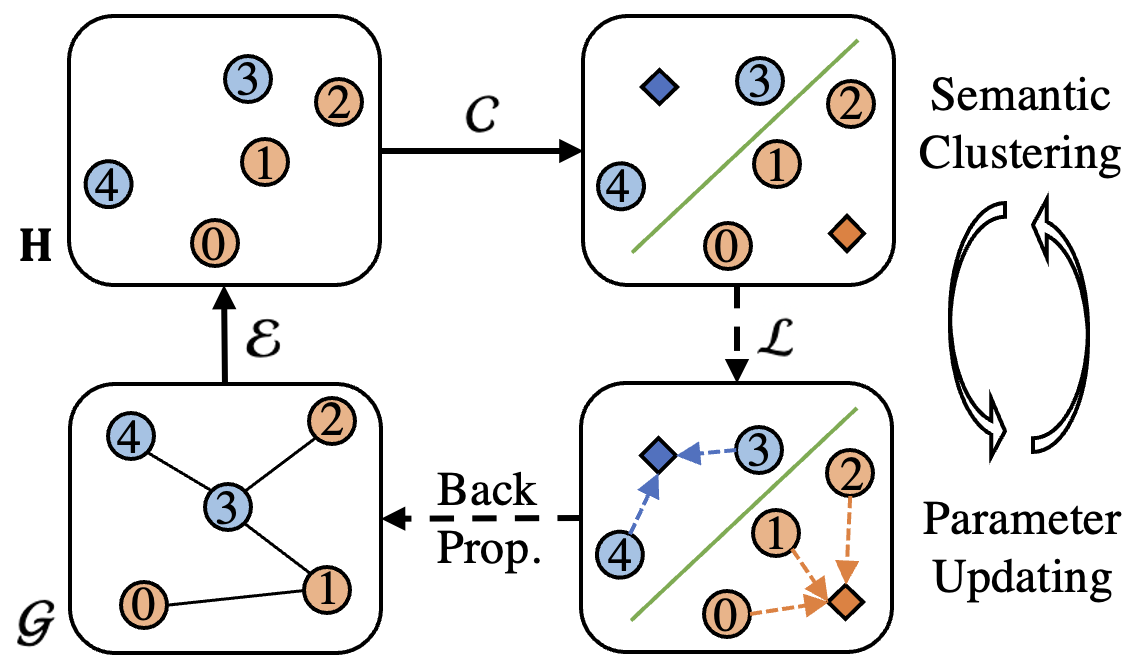}
    \caption{Illustration of GOAL. $\mathcal{E}$ and $\mathcal{C}$ are the encoder and clustering algorithm.
    $\mathcal{G}$ is a homogeneous graph layer and $\mathbf{H}$ is the embedding matrix. 
    $\mathcal{L}$ is given in Equation \eqref{eq:pcl}.
    The circles and diamonds denote nodes and cluster centers.
    Blue and orange denote different hidden semantics.
    The green line is the cluster boundary.
    ``Back Prop.'' means back propagation. 
    The node-level topology based negative sampling treats the semantic similar node 0 and 2 as a negative pair.
    The cluster-level loss reduces semantic error by pulling node 0 and 2 closer to their cluster center.
    }\label{fig:pcl}
\end{figure}

\noindent\textbf{C - Overall Loss.}
Combing the node-level loss in Equation \eqref{eq:node_loss} and the cluster-level loss in Equation \eqref{eq:semantic_loss}, we have:
\begin{equation}\label{eq:pcl}
    \mathcal{L} = \lambda_\mathcal{N}\mathcal{L}_\mathcal{N} + \lambda_\mathcal{C}\mathcal{L}_\mathcal{C}
\end{equation}
where $\lambda_\mathcal{N}$ and $\lambda_\mathcal{C}$ are tunable hyper-parameters.

\subsection{Alignment Regularization}\label{subsec:mvp}
Real-world graphs are often multiplex in nature, which can be decomposed into multiple homogeneous graph layers $\mathcal{G^M}=\{\mathcal{G}^v\}_{v=1}^V$.
The simplest way to extract the embedding of a node $\mathbf{x}_n$ in $\mathcal{G^M}$ is separately extracting the embedding $\{\mathbf{h}_n^v\}_{v=1}^V$ from different layers and then combing them via average pooling.
However, it has been empirically proven that jointly modeling different layers could usually produce better embeddings for downstream tasks \cite{jing2021hdmi}.
Most prior studies use attention modules to jointly learn embeddings from different layers, which are clumsy as they usually require extra efforts to design and train \cite{wang2019heterogeneous,park2020unsupervised,ma2018multi,jing2021hdmi}.
Alternatively, we propose a nimble alignment regularization to jointly learn embeddings by aligning the layer-specific $\{\mathbf{h}_n^v\}_{v=1}^V$ without introducing extra neural network modules, and the final node embedding of $\mathbf{x}_n$ is obtained by simply averaging the layer-specific embeddings $\mathbf{h}_n^\mathcal{M}=\frac{1}{V}\sum_{v=1}^V\mathbf{h}_n^v$.
The underlying assumption of the alignment is that $\mathbf{h}_n^v$ should be close to and reflect the semantics of $\{\mathbf{h}_n^{v'}\}_{v'\neq v}^V$.
The proposed alignment regularization is comprised of both node-level and cluster-level alignments.

Given $\mathcal{G^M}=\{\mathcal{G}^v\}_{v=1}^V$ with encoders $\{\mathcal{E}^v\}_{v=1}^V$, we first apply GOAL to each layer $\mathcal{G}^v$ and obtain the original and negative node embeddings $\{\mathbf{H}^{v}\}_{v=1}^V$ and $\{\mathbf{H}^{v-}\}_{v=1}^V$, as well as the cluster centers $\{\mathbf{C}^v\}_{v=1}^V$, where $\mathbf{C}^v\in\mathbb{R}^{K^v\times d}$ is the concatenation of the cluster centers for the $v$-th layer, $K^v$ is the number of clusters for the $v$-the layer.
The node-level alignment is applied over $\{\mathbf{H}^{v}\}_{v=1}^V$ and $\{\mathbf{H}^{v-}\}_{v=1}^V$.
The cluster-level alignment is used on $\{\mathbf{C}^v\}_{v=1}^V$ and $\{\mathbf{H}^{v}\}_{v=1}^V$.

\noindent\textbf{A - Node-Level Alignment.}\label{subsec:align_node}
For a node $\mathbf{x}_n$, its embedding $\mathbf{h}_n^{v}$ should be close to embeddings $\{\mathbf{h}_n^{v'}\}_{v'\neq v}^V$ and far away from the negative embedding $\mathbf{h}_n^{v-}$.
Analogous to Equation \eqref{eq:node_loss}, we define the node-level alignment regularization as:
\begin{equation}\label{eq:node_align}
    \mathcal{R}_\mathcal{N} = -\frac{1}{Z}\sum_{n=1}^N\sum_{v=1}^V\sum_{v'\neq v}^V\log\frac{\mathbf{e}^{cos(\mathbf{h}_n^v, \mathbf{h}_n^{v'})}}{\mathbf{e}^{cos(\mathbf{h}_n^v, \mathbf{h}_n^{v'})} + \mathbf{e}^{cos(\mathbf{h}_n^v, \mathbf{h}_n^{v-})}}
\end{equation}
where $Z=NV(V-1)$ is the normalization factor.

\begin{figure}
    \centering
    \includegraphics[width=0.35\textwidth]{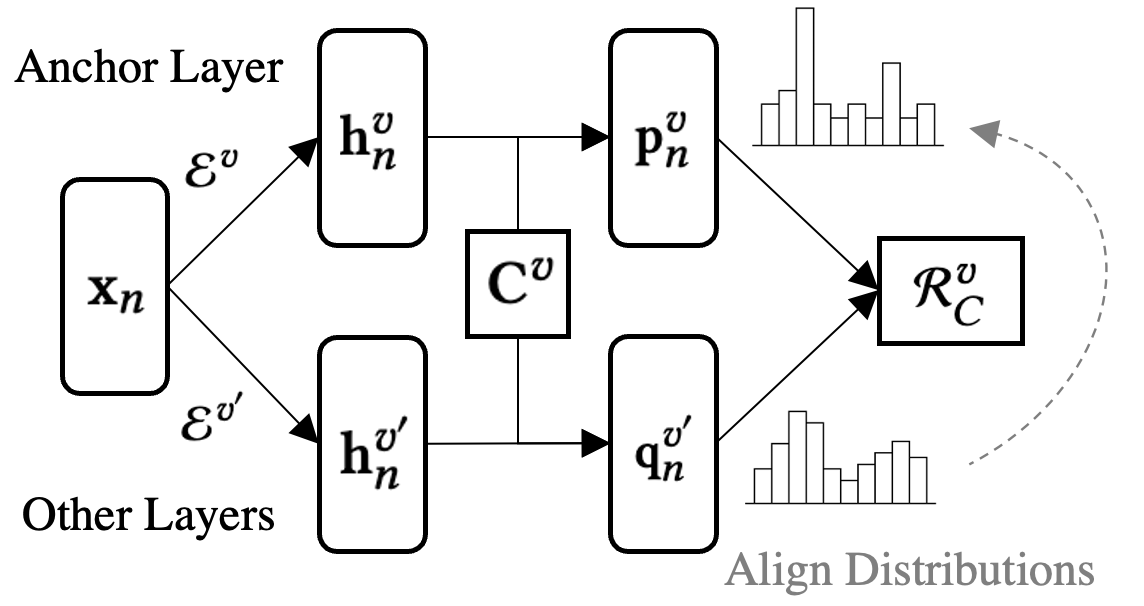}
    \caption{Cluster-level alignment. 
    $\mathbf{x}_n$ is the node attribute.
    $\mathbf{h}_n^v$ and $\mathbf{h}_n^{v'}$ are the layer-specific embeddings.
    $\mathbf{C}^v$ is the anchor cluster center matrix.
    $\mathbf{p}_n^v$ and $\mathbf{q}_n^{v'}$ are the anchor and recovered semantic distributions. $\mathcal{R}_{\mathcal{C}}^v$ is given in Equation \eqref{eq:align_semantic_loss_v}.}\label{fig:semantic_align}
\end{figure}

\noindent\textbf{B - Cluster-Level Alignment.}\label{subsec:align_semantic}
Similar to the node-level loss in Equation \eqref{eq:node_loss}, the node-level alignment in Equation \eqref{eq:node_align} could also introduce semantic errors since $\mathbf{h}_n^{v-}$ might be topologically far from but semantically similar to $\mathbf{h}_n^{v}$.
To reduce the semantic error, we also align the layer-specific embeddings $\{\mathbf{h}_n^v\}_{v=1}^V$ at the cluster level.

Let the $v$-th layer be the anchor layer and its semantic cluster centers $\mathbf{C}^v\in\mathbb{R}^{K^v\times d}$ as the anchor cluster centers.
For a node $\mathbf{x}_n$, we call its layer-specific embedding $\mathbf{h}_n^v$ as the anchor embedding, and its semantic distribution $\mathbf{p}_n^v\in\mathbb{R}^{K^v}$ as the anchor semantics, which is obtained via Equation \eqref{eq:prob} based on $\mathbf{h}_n^v$ and $\mathbf{C}^v$. 
Our key idea of the cluster-level alignment is to recover the anchor semantics $\mathbf{p}_n^v$ from embeddings $\{\mathbf{h}_n^{v'}\}_{v'\neq v}^V$ of other layers based on $\mathbf{C}^v$.

Our idea can be justified from two perspectives.
Firstly, $\{\mathbf{h}_n^{v}\}_{v=1}^V$ reflect information of $\mathbf{x}_n$ from different aspects, if we can recover the anchor semantics $\mathbf{p}^v_n$ from the embedding $\mathbf{h}^{v'}_n$ of another layer $v'\neq v$, then it indicates that $\mathbf{h}^{v}_n$ and $\mathbf{h}^{v'}_n$ share hidden semantics to a certain degree.
Secondly, it is impractical to directly align $\mathbf{p}_n^v$ and $\mathbf{p}_n^{v'}$, since their dimensions might be different $K^v\neq K^{v'}$, and even if $K^v = K^{v'}$, the cluster center vectors $\mathbf{C}^v$ and $\mathbf{C}^{v'}$ are distributed at different positions in the embedding space.




An illustration of the cluster-level alignment is presented in Figure \ref{fig:semantic_align}. 
Given a node $\mathbf{x}_n$, on the anchor layer $v$, we have the anchor cluster centers $\mathbf{C}^v$, the anchor embedding $\mathbf{h}_n^v$, and the anchor semantic distribution $\mathbf{p}_n^v$.
Next, we use the embedding $\mathbf{h}_n^{v'}$ from the layer $v'\neq v$ to obtain the recovered semantic distribution $\mathbf{q}_n^{v'}$ based on $\mathbf{C}^v$ via Equation \eqref{eq:prob}.
Then we align the semantics of $\mathbf{h}_n^{v}$ and $\mathbf{h}_n^{v'}$ by minimizing the KL-divergence of $\mathbf{p}_n^v$ and $\mathbf{q}_n^{v'}$:
\begin{equation}\label{eq:align_semantic_loss_v}
    \mathcal{R}_C^v = \frac{1}{N(V-1)}\sum_{n=1}^N\sum_{v'\neq v}^V KL(\mathbf{p}^v_n||\mathbf{q}^{v'}_n)
\end{equation}
where $\mathbf{p}_n^v$ is treated as the ground-truth and the gradients are not allowed to pass through $\mathbf{p}_n^v$ during training. 



Finally, we alternatively use all $V$ layers as anchor layers and use the averaged KL-divergence as the final semantic regularization:
\begin{equation}\label{eq:align_semantic_loss}
    \mathcal{R}_C = \frac{1}{V}\sum_{v=1}^V\mathcal{R}_C^v
\end{equation}

\noindent\textbf{C - Overall Loss.}
By combining the node-level and cluster-level regularization losses, we have:
\begin{equation}\label{eq:reconstruction}
    \mathcal{R} = \mu_\mathcal{N}\mathcal{R_N} + \mu_\mathcal{C}\mathcal{R_C}
\end{equation}
where $\mu_\mathcal{N}$ and $\mu_\mathcal{C}$ are tunable hyper-parameters.

The final training objective of the X-GOAL framework is the combination of the contrastive loss $\mathcal{L}$ in Equation \eqref{eq:pcl} and the alignment regularization $\mathcal{R}$ in Equation \eqref{eq:reconstruction}: 
\begin{equation}
    \mathcal{L}_{X} = \sum_{v=1}^V\mathcal{L}^v + \mathcal{R}
\end{equation}
where $\mathcal{L}^v$ is the loss of layer $v$



\begin{table*}[t!]
    \small
    \centering
    \caption{Statistics of the datasets}
    \begin{tabular}{c|c|c|c|c|c|c}
    \hline
    Graphs &  \# Nodes & Layers & \# Edges & \# Attributes & \# Labeled Data & \# Classes\\
    \hline
    \multirow{2}{*}{ACM} & \multirow{2}{*}{3,025} & Paper-Subject-Paper (PSP) & 2,210,761 & 1,830 & \multirow{2}{*}{600} & \multirow{2}{*}{3}\\
    & & Paper-Author-Paper (PAP) & 29,281 & (Paper Abstract) & & \\
    \hline
    \multirow{2}{*}{IMDB} & \multirow{2}{*}{3,550} & Movie-Actor-Movie (MAM) & 66,428 & 1,007 & \multirow{2}{*}{300} & \multirow{2}{*}{3}\\
    & & Movie-Director-Movie (MDM) & 13,788 & (Movie plot) & \\
    \hline
    \multirow{3}{*}{DBLP} & \multirow{3}{*}{7,907} & Paper-Author-Paper (PAP) & 144,783 & \multirow{3}{*}{\shortstack{2,000 \\ (Paper Abstract)}} & \multirow{3}{*}{80} & \multirow{3}{*}{4}\\
    & & Paper-Paper-Paper (PPP) & 90,145 & & \\
    & & Paper-Author-Term-Author-Paper (PATAP) & 57,137,515 &  & \\
    \hline
    \multirow{3}{*}{Amazon} & \multirow{3}{*}{7,621} & Item-AlsoView-Item (IVI) & 266,237 & \multirow{3}{*}{\shortstack{2,000 \\ (Item description)}} & \multirow{3}{*}{80} & \multirow{3}{*}{4}\\
    & & Item-AlsoBought-Item (IBI) & 1,104,257 & & \\
    & & Item-BoughtTogether-Item (IOI) & 16,305 &  & \\
    \hline
    \end{tabular}
    \label{tab:datasets}
\end{table*}

\subsection{Theoretical Analysis}\label{subsec:theory}
We provide theoretical analysis for the proposed regularization alignments.
In Theorem \ref{theory_1}, we prove that the node-level alignment maximizes the mutual information of embeddings  $H^{v}\in\{\mathbf{h}_n^{v}\}_{n=1}^N$ of the anchor layer $v$ and embeddings $H^{v'}\in\{\mathbf{h}_n^{v'}\}_{n=1}^N$ of another layer $v'$. 
In Theorem \ref{theory_2}, we prove that the cluster-level alignment maximizes the mutual information of semantic cluster assignments $C^v\in[1, \cdots, K^v]$ for embeddings $\{\mathbf{h}_n^v\}_{n=1}^N$ of the anchor layer $v$ and embeddings $H^{v'}\in\{\mathbf{h}_n^{v'}\}_{n=1}^N$ of the layer $v'$.

\begin{theorem}[Maximization of MI of Embeddings from Different Layers]\label{theory_1}
Let $H^{v}\in\{\mathbf{h}_n^{v}\}_{n=1}^N$ and $H^{v'}\in\{\mathbf{h}_n^{v'}\}_{n=1}^N$ be the random variables for node embeddings of the $v$-th and $v'$-th layers, then the node-level alignment maximizes $I(H^v; H^{v'})$.
\end{theorem}
\begin{proof}
According to \cite{poole2019variational, oord2018representation}, the following inequality holds:
\begin{equation}
    I(X;Y) \geq \mathbb{E}[\frac{1}{K_1}\sum_{i=1}^{K_1}\log\frac{\mathbf{e}^{f(x_i, y_i)}}{\frac{1}{K_2}\sum_{j=1}^{K_2}\mathbf{e}^{f(x_i, y_j)}}]
\end{equation}
Let $K_1=1$, $K_2=2$, $f()=cos()$, $x_1=\mathbf{h}_n^v$, $y_1=\mathbf{h}_n^{v'}$, $y_2=\mathbf{h}_n^{-v}$, then: 
\begin{equation}
    I(H^v; H^{v'}) \geq \mathbb{E}[\log\frac{\mathbf{e}^{cos(\mathbf{h}_n^v, \mathbf{h}_n^{v'})}}{\mathbf{e}^{cos(\mathbf{h}_n^v, \mathbf{h}_n^{v'})} + \mathbf{e}^{cos(\mathbf{h}_n^v, \mathbf{h}_n^{v-})}}]
\end{equation}
The expectation $\mathbb{E}$ is taken over all the $N$ nodes, and all the pairs of $V$ layers, and thus we have:
\begin{equation}
    I(H^v; H^{v'}) \geq \frac{1}{Z}\sum_{n=1}^N\sum_{v=1}^V\sum_{v'\neq v}^V\log\frac{\mathbf{e}^{cos(\mathbf{h}_n^v, \mathbf{h}_n^{v'})}}{\mathbf{e}^{cos(\mathbf{h}_n^v, \mathbf{h}_n^{v'})} + \mathbf{e}^{cos(\mathbf{h}_n^v, \mathbf{h}_n^{v-})}}
\end{equation}
where $Z=NV(V-1)$ is the normalization factor, and the right side is $\mathcal{R_N}$ in Equation \eqref{eq:align_semantic_loss}.
\end{proof}

\begin{theorem}[Maximization of MI between Embeddings and Semantic Cluster Assignments]\label{theory_2}
Let $C^v\in[1, \cdots, K^v]$ be the random variable for cluster assignments for $\{\mathbf{h}^v_{n}\}_{n=1}^N$ of the anchor layer $v$, and $H^{v'}\in\{\mathbf{h}_n^{v'}\}_{n=1}^N$ be the random variable for node embeddings of the $v'$-th layer, then the cluster-level alignment maximizes the mutual information of $C^v$ and $H^{v'}$: $I(C^v; H^{v'})$.
\end{theorem}
\begin{proof}
In the cluster-level alignment, the anchor distribution $\mathbf{p}_n^v$ is regarded as the ground-truth for the $n$-th node, and $\mathbf{q}_{n}^{v'}=f(\mathbf{h}_n^{v'})$ is the recovered distribution from the $v'$-th layer, where $f()$ is a $K^v$ dimensional function defined by Equation \eqref{eq:prob}.
Specifically,
\begin{equation}\label{eq:tmp}
    f(\mathbf{h}_n^{v'})[k]=p(k|\mathbf{h}_n^{v'})=\frac{\mathbf{e}^{(\mathbf{c}_{k}^T\cdot\mathbf{h}_n^{v'}/\tau)}}{\sum_{k'=1}^{K^v}\mathbf{e}^{(\mathbf{c}_{k'}^T\cdot\mathbf{h}_n^{v'}/\tau)}}
\end{equation}
where $\{\mathbf{c}_{k}\}_{k=1}^{K^v}$ is the set of cluster centers for the $v$-th layer.

Since $\mathbf{p}_n^v$ is the ground-truth, and thus its entropy $H(\mathbf{p}_n^v)$ is a constant.
As a result, the KL divergence in Equation \eqref{eq:align_semantic_loss_v} is equivalent to cross-entropy $H(\mathbf{p}_n^v,\mathbf{q}_{n}^{v'})=KL(\mathbf{p}_n^v||\mathbf{q}_{n}^{v'}) + H(\mathbf{p}_n^v)$.
Therefore, minimizing the KL-divergence will minimize $H(\mathbf{p}_n^v,\mathbf{q}_{n}^{v'})$.

On the other hand, according to \cite{mcallester2020formal, qin2019rethinking}, we have the following variational lower bound for $I(C^v;H^{v'})$:
\begin{equation}
    I(C^v;H^{v'})\geq\mathbb{E}[\log\frac{\mathbf{e}^{g(\mathbf{h}_n^{v'}, k)}}{\sum_{k'=1}^{K^v}\mathbf{e}^{g(\mathbf{h}_n^{v'}, k')}}]
\end{equation}
where $g()$ is any function of $\mathbf{h}_n^{v'}$ and $k$.

In our case, we let
\begin{equation}
    g(\mathbf{h}_n^{v'}, k) = \frac{1}{\tau}\mathbf{c}_{k}^T\cdot\mathbf{h}_n^{v'}
\end{equation}
where $\mathbf{c}_{k}$ is the $k$-th semantic cluster center of the $v$-th layer, and $\tau$ is the temperature parameter.

As a result, we have
\begin{equation}
    \frac{\mathbf{e}^{g(\mathbf{h}_n^{v'}, k)}}{\sum_{k'=1}^{K^v}\mathbf{e}^{g(\mathbf{h}_n^{v'}, k')}} = f[\mathbf{h}_n^{v'}][k] = \mathbf{q}_n^{v'}[k]
\end{equation}

The expectation $\mathbb{E}$ is taken over the ground-truth distribution of the cluster assignments for the anchor layer $v$:
\begin{equation}
    p_{gt}(\mathbf{h}_n^{v'}, k) = p_{gt}(\mathbf{h}_n^{v'})p_{gt}(k|\mathbf{h}_n^{v'}) = \frac{1}{N}\mathbf{p}_n^{v}[k]
\end{equation}
where $p_{gt}(k|\mathbf{h}_n^{v'})=\mathbf{p}_n^{v}[k]$ is the ground-truth semantic distribution for $\mathbf{h}_n^{v'}$ on the anchor layer $v$, which is different from the recovered distribution $p(k|\mathbf{h}_n^{v'})=\mathbf{q}_n^{v'}[k]$ shown in Equation \eqref{eq:tmp}.

Therefore, we have
\begin{equation}
    I(C^v;H^{v'})\geq\frac{1}{Z}\sum_{n=1}^N\sum_{k=1}^{K^v}\mathbf{p}_n^v[k]\log\mathbf{q}_n^{v'}[k]=-\frac{1}{Z}\sum_{n=1}^NH(\mathbf{p}_n^v,\mathbf{q}_{n}^{v'})
\end{equation}
where $Z=NK^v$ is the normalization factor.

Thus, minimizing $H(\mathbf{p}_n^v,\mathbf{q}_{n}^{v'})$ will maximize $I(C^v;H^{v'})$.
\end{proof}

\section{Experiments}\label{sec:experiments}

\subsection{Experimental Setups}

\noindent\textbf{Datasets.} 
We use publicly available multiplex heterogeneous graph datasets \cite{park2020unsupervised, jing2021hdmi}: ACM, IMDB, DBLP and Amazon to evaluate the proposed methods.
The statistics is summarized in Table \ref{tab:datasets}.

\begin{table*}[t]
    \small
    \centering
    \caption{Overall performance of X-GOAL on the supervised task: node classification.}
    \begin{tabular}{l|cc|cc|cc|cc}
         \hline
         Dataset & \multicolumn{2}{c|}{ACM} & \multicolumn{2}{c|}{IMDB} & \multicolumn{2}{c|}{DBLP} & \multicolumn{2}{c}{Amazon} \\
         \hline
         Metric & Macro-F1 & Micro-F1 & Macro-F1 & Micro-F1 & Macro-F1 & Micro-F1 & Macro-F1 & Micro-F1\\
         \hline
         DeepWalk & 0.739 & 0.748 & 0.532 & 0.550 & 0.533 & 0.537  & 0.663 & 0.671 \\
         node2vec & 0.741 & 0.749 & 0.533 & 0.550 & 0.543 & 0.547 & 0.662 & 0.669\\
         GCN/GAT & 0.869 & 0.870 & 0.603 & 0.611 & 0.734 & 0.717 & 0.646 & 0.649\\
         DGI & 0.881 & 0.881 & 0.598 & 0.606 & 0.723 & 0.720 & 0.403 & 0.418\\
         ANRL & 0.819 & 0.820 & 0.573 & 0.576 & 0.770 & 0.699 & 0.692 & 0.690\\
         CAN & 0.590 & 0.636 & 0.577 & 0.588 & 0.702 & 0.694 & 0.498 & 0.499\\
         DGCN & 0.888 & 0.888 & 0.582 & 0.592 & 0.707&  0.698 & 0.478 & 0.509\\
         GraphCL  & 0.884 &0.883 & 0.619 & 0.623 & 0.814 & 0.806& 0.461 & 0.472\\
         GCA & 0.798 & 0.797 & 0.523 & 0.533 & OOM & OOM & 0.408 & 0.398\\
         HDI & {0.901} & 0.900 & 0.634 & 0.638 & 0.814 & 0.800 & 0.804 & 0.806 \\
         \hline
         CMNA & 0.782 & 0.788 & 0.549 & 0.566 & 0.566 & 0.561 & 0.657 & 0.665\\
         MNE & 0.792 & 0.797 & 0.552 & 0.574 & 0.566 & 0.562 & 0.556 & 0.567\\
         mGCN & 0.858 & 0.860 & 0.623 & 0.630 & 0.725 & 0.713 & 0.660 & 0.661\\
         HAN & 0.878 & 0.879 & 0.599 & 0.607 & 0.716 & 0.708 & 0.501 & 0.509\\
         DMGI & 0.898 & 0.898 & 0.648 & 0.648 & 0.771 & 0.766 & 0.746 & 0.748\\
         DMGI$_{\text{attn}}$ & 0.887 & 0.887 & 0.602 & 0.606 & 0.778 & 0.770 & 0.758 & 0.758\\
         MvAGC & 0.778 & 0.791 & 0.598 & 0.615 & 0.509 & 0.542 & 0.395 & 0.414\\
         HDMI & {0.901} & {0.901} & {0.650} & {0.658} & {0.820} & {0.811} & {0.808} & {0.812}\\
         \hline
         X-GOAL & \textbf{0.922} & \textbf{0.921} & \textbf{0.661} & \textbf{0.663} & \textbf{0.830} & \textbf{0.819} & \textbf{0.858} & \textbf{0.857}\\
         \hline
    \end{tabular}
    \label{tab:overall_classification}
\end{table*}

\begin{table*}[t]
    \small
    \centering
    \caption{Overall performance of X-GOAL on the unsupervised tasks: node clustering and similarity search.
    }
    \begin{tabular}{l|cc|cc|cc|cc}
         \hline
         Dataset & \multicolumn{2}{c|}{ACM} & \multicolumn{2}{c|}{IMDB} & \multicolumn{2}{c|}{DBLP} & \multicolumn{2}{c}{Amazon} \\
         \hline
         Metric & NMI & Sim@5 & NMI & Sim@5 & NMI & Sim@5 & NMI & Sim@5\\
         \hline
         DeepWalk & 0.310 & 0.710 & 0.117 & 0.490 & 0.348 & 0.629  & 0.083 & 0.726 \\
         node2vec & 0.309 & 0.710 & 0.123 & 0.487 & 0.382 & 0.629 & 0.074 & 0.738\\
         GCN/GAT & 0.671 & 0.867 & 0.176 & 0.565 & 0.465 & 0.724 & 0.287 & 0.624\\
         DGI & 0.640 & 0.889 & 0.182 & 0.578 & 0.551 & 0.786 & 0.007 & 0.558\\
         ANRL & 0.515 & 0.814 & 0.163 & 0.527 & 0.332 & 0.720 & 0.166 & 0.763\\
         CAN & 0.504 & 0.836 & 0.074 & 0.544 & 0.323 & 0.792 & 0.001 & 0.537\\
         DGCN & 0.691 & 0.690 & 0.143 & 0.179 & 0.462 & 0.491 & 0.143 & 0.194\\
         GraphCL & 0.673 & 0.890 & 0.149 & 0.565 & 0.545 & 0.803 & 0.002&0.360\\
         GCA & 0.443 & 0.791 & 0.007 & 0.496 & OOM & OOM &0.002 & 0.478\\
         HDI & 0.650 & 0.900 & 0.194 & 0.605 & 0.570 & 0.799 & 0.487 & 0.856 \\
         \hline
         CMNA & 0.498 & 0.363 & 0.152 & 0.069 & 0.420 & 0.511 & 0.070 & 0.435\\
         MNE & 0.545 & 0.791 & 0.013 & 0.482 & 0.136 & 0.711 & 0.001 & 0.395\\
         mGCN & 0.668 & 0.873 & 0.183 & 0.550 & 0.468 & 0.726 & 0.301 & 0.630\\
         HAN & 0.658 & 0.872 & 0.164 & 0.561 & 0.472 & 0.779 & 0.029 & 0.495\\
         DMGI & 0.687 & 0.898 & 0.196 & 0.605 & 0.409 & 0.766 & 0.425 & 0.816\\
         DMGI$_{\text{attn}}$ & 0.702 & 0.901 & 0.185 & 0.586 & 0.554 & 0.798 & 0.412 & 0.825\\
         MvAGC & 0.665 & 0.824 & 0.219 & 0.525 & 0.281 & 0.437 &0.082 & 0.237\\
         HDMI & 0.695 & 0.898 & 0.198 & 0.607 & 0.582 & \textbf{0.809} & 0.500 & 0.857\\
         \hline
         X-GOAL & \textbf{0.773} & \textbf{0.924} & \textbf{0.221} & \textbf{0.613} & \textbf{0.615} & \textbf{0.809} & \textbf{0.556} & \textbf{0.907}\\
         \hline
    \end{tabular}
    \label{tab:overall_clustering}
\end{table*}

\noindent\textbf{Comparison Methods.}
We compare with methods for
\textit{(1) attributed graphs}, including methods disregarding node attributes: DeepWalk \cite{perozzi2014deepwalk} and node2vec \cite{grover2016node2vec}, 
and methods considering attributes: 
GCN \cite{kipf2016semi}, GAT \cite{velivckovic2017graph}, DGI \cite{velivckovic2018deep}, ANRL \cite{zhang2018anrl}, CAN \cite{meng2019co}, DGCN \cite{zhuang2018dual}, HDI\cite{jing2021hdmi}, GCA \cite{zhu2021graph} and GraphCL \cite{you2020graph};
\textit{(2) attributed multiplex heterogeneous graphs}, including methods disregarding node attributes:
CMNA \cite{chu2019cross}, MNE \cite{zhang2018scalable},
and methods considering attributes:
mGCN \cite{ma2019multi}, HAN \cite{wang2019heterogeneous}, MvAGC \cite{lin2021graph},
DMGI, DMGI$_{\text{attn}}$ \cite{park2020unsupervised} and HDMI \cite{jing2021hdmi}.

\noindent\textbf{Evaluation Metrics.}
Following \cite{jing2021hdmi}, 
we first extract embeddings from the trained encoder. 
Then we train downstream models with the extracted embeddings, and evaluate models' performance on the following tasks:
\textit{(1) a supervised task}: node classification;
\textit{(2) unsupervised tasks}: node clustering and similarity search.
For the node classification task, we train a logistic regression model and evaluate its performance with Macro-F1 (MaF1) and Micro-F1 (MiF1).
For the node clustering task, we train the K-means algorithm and evaluate it with Normalized Mutual Information (NMI).
For the similarity search task, we first calculate the cosine similarity for each pair of nodes, and for each node, we compute the rate of the nodes to have the same label within its 5 most similar nodes (Sim@5).

\noindent\textbf{Implementation Details.}
We use the one layer 1st-order GCN \cite{kipf2016semi} with tangent activation as the encoder $\mathcal{E}^v=\text{tanh}(\mathbf{A}^v\mathbf{X}\mathbf{W} + \mathbf{X}\mathbf{W}' + \mathbf{b})$.
We set dimension $d=128$ and $p_{drop}=0.5$.
The models are implemented by PyTorch \cite{paszke2019pytorch} and trained on NVIDIA Tesla V-100 GPU.
During training, we first warm up the encoders by training them with the node-level losses $\mathcal{L_N}$ and $\mathcal{R_N}$.
Then we apply the overall loss $\mathcal{L_X}$ with the learning rate of 0.005 for IMDB and 0.001 for other datasets.
We use K-means as the clustering algorithm, and the semantic clustering step is performed every 5 epochs of parameter updating.
We adopt early stopping with the patience of 100 to prevent overfitting.

\subsection{Overall Performance}

\noindent\textbf{X-GOAL on Multiplex Heterogeneous Graphs.}
The overall performance for all of the methods is presented in Tables \ref{tab:overall_classification}-\ref{tab:overall_clustering}, where the upper and middle parts are the methods for homogeneous graphs and multiplex heterogeneous graphs respectively.
``OOM'' means out-of-memory.
Among all the baselines, HDMI has the best overall performance.
The proposed X-GOAL further outperforms HDMI.
The proposed X-GOAL has 0.023/0.019/0.041/0.021 average improvements over the second best scores on Macro-F1/Micro-F1/NMI/Sim@5.
For Macro-F1 and Micro-F1 in Table \ref{tab:overall_classification}, X-GOAL improves the most on the Amazon dataset (0.050/0.044).
For NMI and Sim@5 in Table \ref{tab:overall_clustering}, X-GOAL improves the most on the ACM (0.071) and Amazon (0.050) dataset respectively.
The superior overall performance of X-GOAL demonstrate that the proposed approach can effectively extract informative node embeddings for multiplex heterogeneous graph.

\begin{table*}[t!]
    \scriptsize
    \centering
    \setlength\tabcolsep{4.5pt} 
    \caption{Overall performance of GOAL on each layer: node classification.}
    \begin{tabular}{l|cc|cc|cc|cc|cc|cc|cc|cc|cc|cc}
        \hline
        Dataset & \multicolumn{4}{c|}{ACM} & \multicolumn{4}{c|}{IMDB} & \multicolumn{6}{c|}{DBLP} & \multicolumn{6}{c}{Amazon} \\
        \hline
        View & \multicolumn{2}{c|}{PSP} & \multicolumn{2}{c|}{PAP} & \multicolumn{2}{c|}{MDM} & \multicolumn{2}{c|}{MAM} & \multicolumn{2}{c|}{PAP} & \multicolumn{2}{c|}{PPP} & \multicolumn{2}{c|}{PATAP} & \multicolumn{2}{c|}{IVI} & \multicolumn{2}{c|}{IBI} & \multicolumn{2}{c}{IOI}\\
        \hline
        Metric & MaF1 & MiF1 & MaF1 & MiF1 & MaF1 & MiF1 & MaF1 & MiF1 & MaF1 & MiF1 & MaF1 & MiF1 & MaF1 & MiF1 & MaF1 & MiF1 & MaF1 & MiF1 & MaF1 & MiF1\\
        \hline
        DGI & 0.663 & 0.668 & 0.855 & 0.853 & 0.573 & 0.586 & 0.558 & 0.564 & 0.804 & 0.796 & 0.728 & 0.717 & 0.240 & 0.272 & 0.380 & 0.388 & 0.386 & 0.410 & 0.569 & 0.574\\
        GraphCL & 0.649 & 0.658 & 0.833 & 0.824 & 0.551 & 0.566 & 0.554 & 0.562 & 0.806 & 0.779 & 0.678 & 0.675 & 0.236 & 0.286 & 0.290 & 0.305 &0.335 & 0.348 & 0.506 & 0.516 \\
        GCA & 0.645 & 0.656 & 0.748 & 0.749 & 0.534 & 0.537 &0.489 & 0.500 & 0.716 & 0.710 & 0.679 & 0.665 & OOM & OOM &0.300 & 0.312&0.289&0.304&0.532&0.526\\
        HDI & {0.742} & {0.744} & {0.889} & {0.888} & {0.626} & {0.631} & {0.600} & {0.606} & {0.812} & {0.803} & {0.751} & {0.745} & {0.241} & {0.284} & {0.581} & {0.583} &  {0.524} & {0.529} & {0.796} & {0.799}\\
        \hline
        GOAL & \textbf{0.833} & \textbf{0.836} & \textbf{0.908} & \textbf{0.908} & \textbf{0.649} & \textbf{0.653} & \textbf{0.653} & \textbf{0.652} & \textbf{0.817} & \textbf{0.804} & \textbf{0.765}  & \textbf{0.755} & \textbf{0.755} & \textbf{0.745} & \textbf{0.849} & \textbf{0.848} & \textbf{0.850} & \textbf{0.848} & \textbf{0.851} & \textbf{0.851}\\
        \hline
    \end{tabular}
    \label{tab:view_classification}
\end{table*}

\begin{table*}[t!]
    \scriptsize
    \centering
    \setlength\tabcolsep{4pt} 
    \caption{Overall performance of GOAL on each layer: node clustering and similarity search.}
    \begin{tabular}{l|cc|cc|cc|cc|cc|cc|cc|cc|cc|cc}
        \hline
        Dataset & \multicolumn{4}{c|}{ACM} & \multicolumn{4}{c|}{IMDB} & \multicolumn{6}{c|}{DBLP} & \multicolumn{6}{c}{Amazon} \\
        \hline
        View & \multicolumn{2}{c|}{PSP} & \multicolumn{2}{c|}{PAP} & \multicolumn{2}{c|}{MDM} & \multicolumn{2}{c|}{MAM} & \multicolumn{2}{c|}{PAP} & \multicolumn{2}{c|}{PPP} & \multicolumn{2}{c|}{PATAP} & \multicolumn{2}{c|}{IVI} & \multicolumn{2}{c|}{IBI} & \multicolumn{2}{c}{IOI}\\
        \hline
        Metric & NMI & Sim@5 & NMI & Sim@5 & NMI & Sim@5 & NMI & Sim@5 & NMI & Sim@5 & NMI & Sim@5 & NMI & Sim@5 & NMI & Sim@5 & NMI & Sim@5 & NMI & Sim@5\\
        \hline
        DGI & 0.526 & 0.698 & 0.651 & 0.872 & 0.145 & 0.549 & 0.089 & 0.495 & 0.547 & 0.800 & 0.404 & 0.741 & 0.054 & 0.583 & 0.002 & 0.395 & 0.003 & 0.414 & 0.038 & 0.701\\
        GraphCL & 0.524 & 0.735 & 0.675 & 0.874 &  0.128 & 0.554 & 0.060 & 0.485 & 0.539 & 0.794 & 0.347 & 0.702 & 0.052 & 0.595 &0.001 & 0.334 & 0.002 & 0.360 & 0.036&0.630\\
        GCA & 0.389 & 0.662 & 0.062 & 0.764 & 0.008 & 0.491 & 0.008 & 0.463 & 0.076 & 0.775 & 0.223 & 0.683 & OOM & OOM & 0.002 & 0.315 & 0.007 & 0.329 & 0.008&0.588\\
        HDI & {0.528} & 0.716 & {0.662} & {0.886} & {0.194} & 0.592 & {0.143} & {0.527} & 0.562 & {0.805} & 0.408 & {0.742} & {0.054} & {0.591} & {0.169} & {0.544} & {0.153} & {0.525} & {0.407} & {0.826}\\
        \hline
        GOAL & \textbf{0.600} & \textbf{0.851} & \textbf{0.735} & \textbf{0.917} & \textbf{0.210} & \textbf{0.602} & \textbf{0.180} & \textbf{0.585} & \textbf{0.589} & \textbf{0.809} & \textbf{0.447}  & \textbf{0.757} & \textbf{0.412} & \textbf{0.733} & \textbf{0.551} & \textbf{0.901} & \textbf{0.544} & \textbf{0.903} & \textbf{0.536} & \textbf{0.905}\\
        \hline
    \end{tabular}
    \label{tab:view_clustering}
\end{table*}

\begin{table*}[t!]
    \scriptsize
    \centering
    \caption{Ablation study of X-GOAL at the multiplex heterogeneous graph level.}
    \begin{tabular}{l|cccc|cccc|cccc|cccc}
        \hline
        Dataset & \multicolumn{4}{c|}{ACM} & \multicolumn{4}{c|}{IMDB} & \multicolumn{4}{c|}{DBLP} & \multicolumn{4}{c}{Amazon} \\
        \hline
        Metric & MaF1 & MiF1 & NMI & Sim@5 & MaF1 & MiF1 & NMI & Sim@5 & MaF1 & MiF1 & NMI & Sim@5 & MaF1 & MaF1 & MiF1 & Sim@5 \\
        \hline
        X-GOAL & \textbf{0.922} & \textbf{0.921} & \textbf{0.773} & \textbf{0.924} & \textbf{0.661} & \textbf{0.663} & \textbf{0.221} & \textbf{0.613} & \textbf{0.830} & \textbf{0.819} & \textbf{0.615}  & \textbf{0.809} & \textbf{0.858} & \textbf{0.857} & \textbf{0.556} & \textbf{0.907} \\
        w/o $\mathcal{R}_C$ & 0.919 & 0.917 & 0.770 & 0.922 & 0.658 & 0.661 & 0.211 & 0.606 & 0.817 & 0.807 & 0.611 & 0.804 & 0.856 & 0.856 & 0.555 & 0.906\\
        w/o $\mathcal{R}_N$, $\mathcal{R}_C$ & 0.893 & 0.893 & 0.724 & 0.912 & 0.651 &0.658 & 0.194 & 0.606 & 0.803 & 0.791 & 0.590 & 0.801 & 0.835 & 0.834 & 0.506 & 0.904 \\
        \hline
    \end{tabular}
    \label{tab:ablation_multiplex}
\end{table*}

\noindent\textbf{GOAL on Homogeneous Graph Layers.} 
We compare the proposed GOAL framework with recent infomax-based methods (DGI and HDI) and graph augmentation based methods (GraphCL and GCA).
The experimental results for each single homogeneous graph layer are presented in Tables \ref{tab:view_classification}-\ref{tab:view_clustering}.
It is evident that GOAL significantly outperforms the baseline methods on all single homogeneous graph layers.
On average, GOAL has 0.137/0.129/0.151/0.119 improvements on Macro-F1/Micro-F1/NMI/Sim@5.
For node classification in Table \ref{tab:view_classification}, GOAL improves the most on the PATAP layer of DBLP: 0.514/0.459 on Macro-F1/Micro-F1.
For node clustering and similarity search in Table \ref{tab:view_clustering}, GOAL improves the most on the IBI layer of Amazon: 0.391 on NMI and 0.378 on Sim@5.
The superior performance of GOAL indicates that the proposed prototypical contrastive learning strategy is better than the infomax-based and graph augmentation based instance-wise contrastive learning strategies.
We believe this is because prototypical contrasive learning could effectively reduce the semantic errors.

\subsection{Ablation Study}

\noindent\textbf{Multiplex Heterogeneous Graph Level.}
In Table \ref{tab:ablation_multiplex}, we study the impact of the node-level and semantic-level alignments.
The results in Table \ref{tab:ablation_multiplex} indicate that both of the node-level alignment ($\mathcal{R}_N$) and the semantic-level alignment ($\mathcal{R}_C$) can improve the performance.

\noindent\textbf{Homogeneous Graph Layer Level.}
The results for different configurations of GOAL on the PAP layer of ACM are shown in Table \ref{tab:ablation_study}.
First, all of the warm-up, the semantic-level loss $\mathcal{L_C}$ and the node-level loss $\mathcal{L_N}$ are critical.
Second, comparing GOAL (1st-order GCN with tanh activation) with other GCN variants, (1) with the same activation function, the 1st-order GCN perform better than the original GCN; (2) tanh is better than relu.
We believe this is because the 1st-order GCN has a better capability for capturing the attribute information, and tanh provides a better normalization for the node embeddings.
Finally, for the configurations of graph transformation, if we replace dropout with masking, the performance will drop.
This is because dropout re-scales the outputs by ${1}/{(1-p_{drop})}$, which improves the performance.
Besides, dropout on both attributes and adjacency matrix is important.

\begin{table}[t!]
    \small
    \centering
    \caption{Ablation study of GOAL on the PAP layer of ACM.}
    \setlength\tabcolsep{5pt} 
    \begin{tabular}{l|cccc}
        \hline
         & MaF1 & MiF1 & NMI & Sim@5\\
        \hline
        GOAL &  \textbf{0.908} & \textbf{0.908} & \textbf{0.735} & \textbf{0.917}\\
        \hline
        w/o warm-up & 0.863 & 0.865 & 0.721 & 0.903 \\
        w/o $\mathcal{L_C}$ & 0.865 & 0.867 & 0.693 & 0.899 \\
        w/o $\mathcal{L_N}$ & 0.878 & 0.880 & 0.678 & 0.881 \\
        \hline
        1st-ord. GCN (relu) & 0.865 & 0.866 & 0.559 & 0.859\\
        GCN (tanh) & 0.881 & 0.881 & 0.486 & 0.886 \\
        GCN (relu) & 0.831 & 0.831 & 0.410 & 0.837 \\
        \hline
        dropout $\rightarrow$ masking & 0.888 & 0.890 & 0.716 & 0.903 \\
        w/o attribute drop & 0.843 & 0.845 & 0.568 & 0.869 \\
        w/o adj. matrix drop & 0.888 & 0.888 & 0.715 & 0.903\\
        \hline
    \end{tabular}
    \label{tab:ablation_study}
\end{table}

\begin{figure}[t!]
    \centering
    \begin{subfigure}[b]{.23\textwidth}
        \includegraphics[width=1\textwidth]{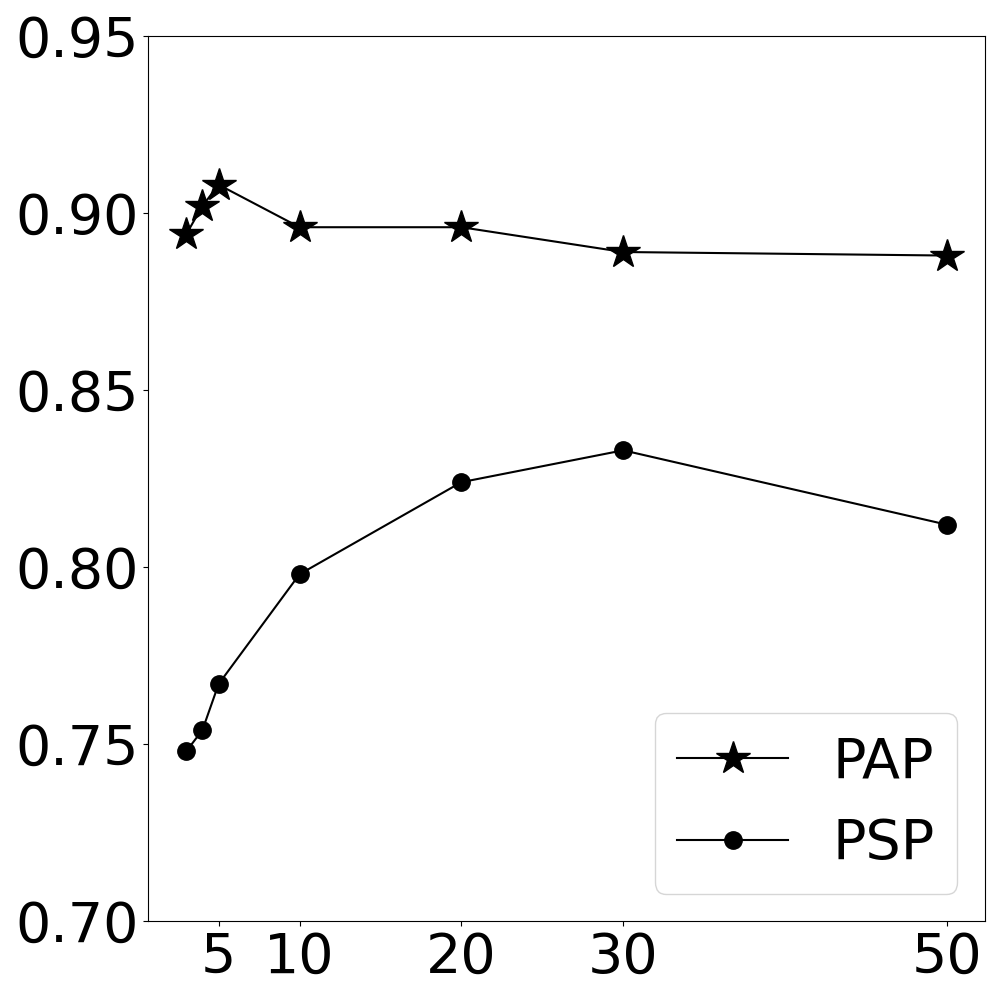}
        \caption{Macro-F1 v.s. $K$}\label{fig:k_f1}
    \end{subfigure}
    \begin{subfigure}[b]{.23\textwidth}
        \includegraphics[width=1\textwidth]{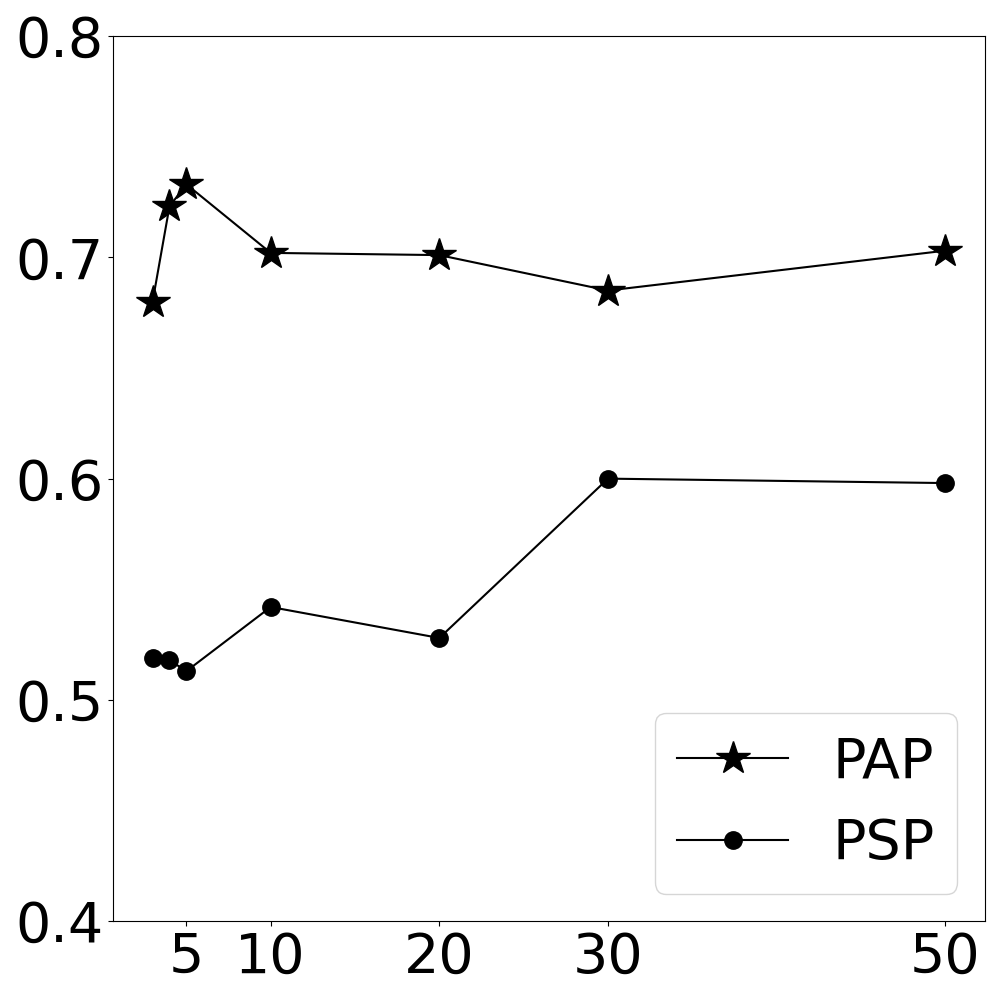}
        \caption{NMI v.s. $K$}\label{fig:k_nmi}
    \end{subfigure}
    \caption{The number of $K$ on PSP and PAP of ACM}
    \label{fig:K}
\end{figure}

\subsection{Number of Clusters}
Figure \ref{fig:K} shows the Macro-F1 and NMI scores on the PSP and PAP layers of ACM w.r.t. the number of clusters $K\in[3,4,5,10,20, 30,50]$.
For PSP and PAP, the best Macro-F1 and NMI scores are obtained when $K=30$ and $K=5$.
The number of ground-truth classes for ACM is 3, and the results in Figure \ref{fig:K} indicate that over-clustering is beneficial. 
We believe this is because there are many sub-clusters in the embedding space, which is consistent with the prior findings on image data \cite{li2020prototypical}.

\begin{figure*}[t!]
    \centering
    \begin{subfigure}[b]{.245\textwidth}
        \includegraphics[width=1\textwidth]{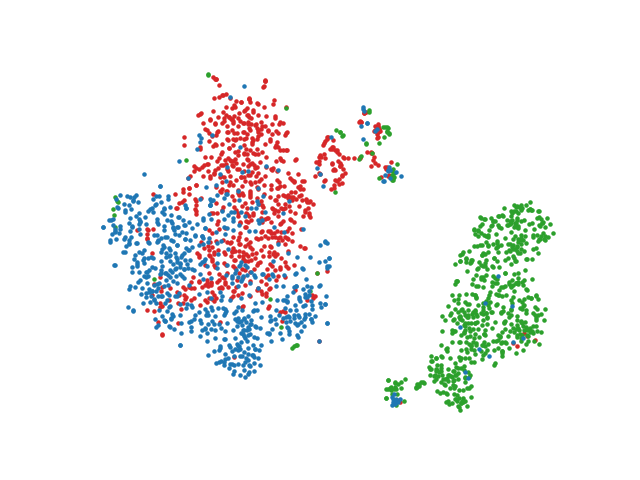}
        \caption{$\mathcal{L_N}$ on PSP}\label{fig:tsne_wp_1}
    \end{subfigure}
    \begin{subfigure}[b]{.245\textwidth}
        \includegraphics[width=1\textwidth]{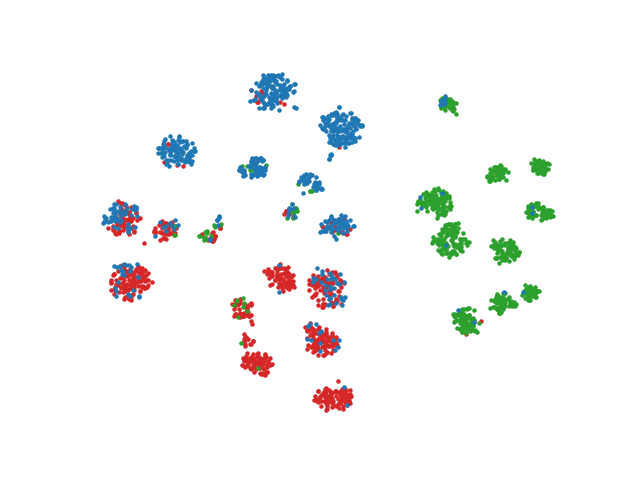}
        \caption{$\mathcal{L_N}+\mathcal{L_C}$ on PSP}\label{fig:tsne_pcl_1}
    \end{subfigure}
    \begin{subfigure}[b]{.245\textwidth}
        \includegraphics[width=1\textwidth]{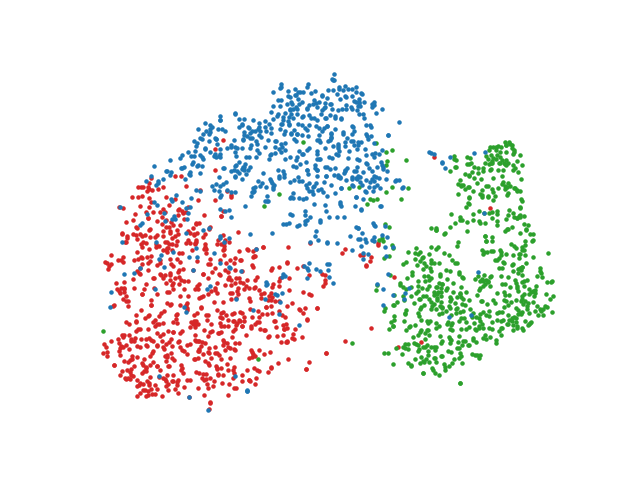}
        \caption{$\mathcal{L_N}$ on PAP}\label{fig:tsne_wp_2}
    \end{subfigure}
    \begin{subfigure}[b]{.245\textwidth}\label{fig:tsne_pcl_2}
        \includegraphics[width=1\textwidth]{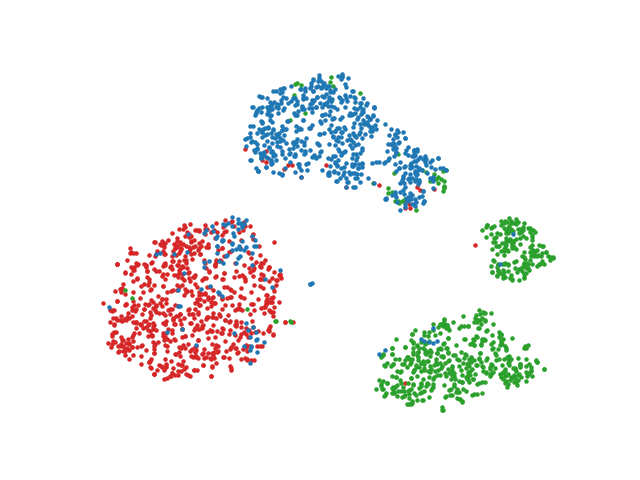}
        \caption{$\mathcal{L_N}+\mathcal{L_C}$ on PAP}
    \end{subfigure}
    \caption{Visualization of the embeddings for the PAP and PSP layers of the ACM graph.}
    \label{fig:visualization_views}
\end{figure*}

\begin{figure*}[t!]
    \centering
    \begin{subfigure}[b]{.245\textwidth}
        \includegraphics[width=1\textwidth]{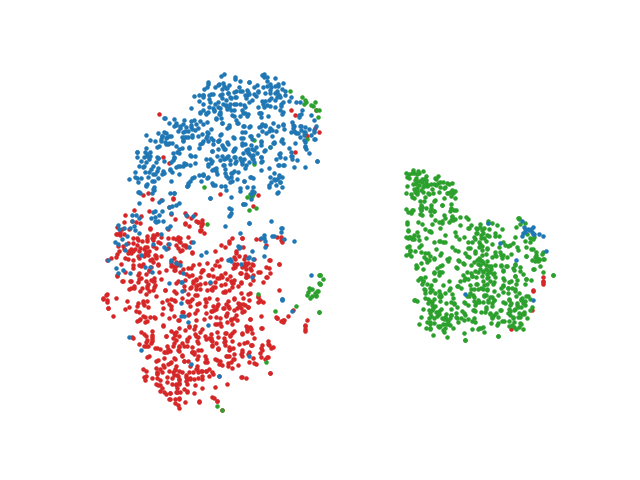}
        \caption{$\mathcal{L_N}$}\label{fig:tsne_wp_avg}
    \end{subfigure}
    \begin{subfigure}[b]{.245\textwidth}
        \includegraphics[width=1\textwidth]{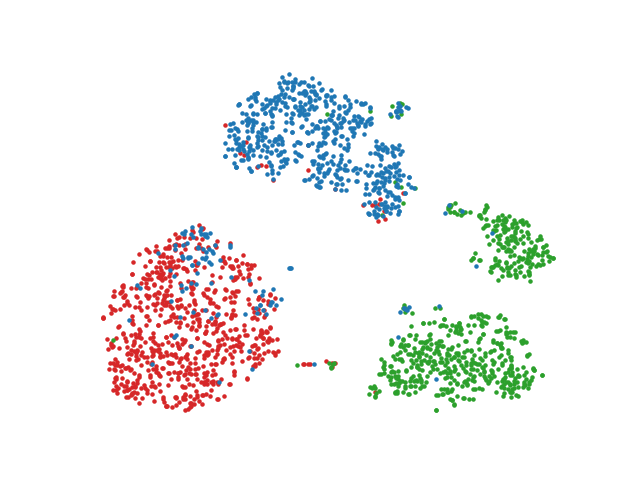}
        \caption{$\mathcal{L_N}+\mathcal{L_C}$}\label{fig:tsne_pcl_avg}
    \end{subfigure}
    \begin{subfigure}[b]{.245\textwidth}
        \includegraphics[width=1\textwidth]{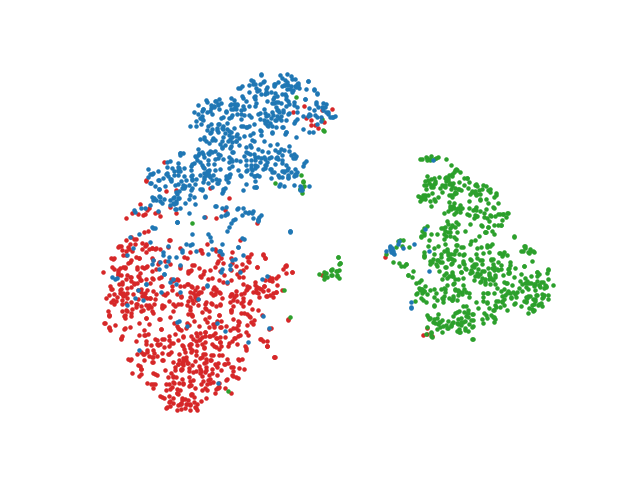}
        \caption{$\mathcal{L_N}+\mathcal{L_C}+\mathcal{R_N}$}\label{fig:tsne_no_reg}
    \end{subfigure}
    \begin{subfigure}[b]{.245\textwidth}
        \includegraphics[width=1\textwidth]{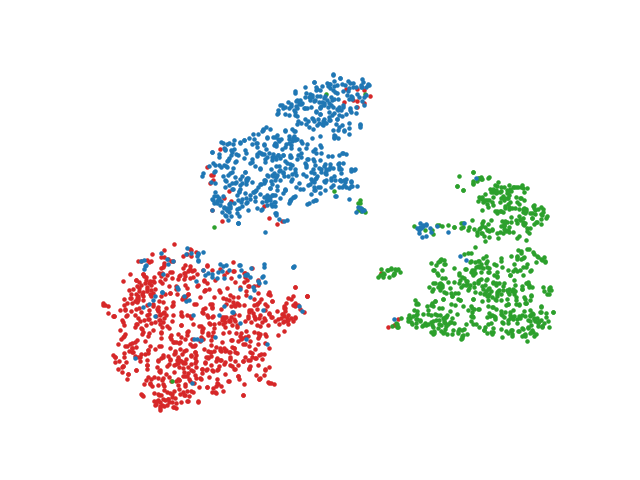}
        \caption{$\mathcal{L_N}+\mathcal{L_C}+\mathcal{R_N}+\mathcal{R_C}$}\label{fig:tsne_full}
    \end{subfigure}
    \caption{Visualization of the combined embeddings for the ACM graph.}
    \label{fig:visualization_multiplex}
\end{figure*}

\subsection{Visualization}

\noindent\textbf{Homogeneous Graph Layer Level.} The t-SNE \cite{maaten2008visualizing} visualizations of the embeddings for PSP and PAP of ACM are presented in Figure \ref{fig:visualization_views}.
$\mathcal{L_N}$, $\mathcal{L_C}$, $\mathcal{R_N}$ and $\mathcal{R_C}$ are the node-level loss, cluster-level loss, node-level alignment and cluster-level alignment.
The embeddings extracted by the full GOAL framework ($\mathcal{L_N}+\mathcal{L_C}$) are better separated than the node-level loss $\mathcal{L_N}$ only.
For GOAL, the numbers of clusters for PSP and PAP are 30 and 5 since they have the best performance as shown in Figure \ref{fig:K}.

\noindent\textbf{Multiplex Heterogeneous Graph Level.}
The visualizations for the combined embeddings are shown in Figure \ref{fig:visualization_multiplex}.
Embeddings in Figures \ref{fig:tsne_wp_avg}-\ref{fig:tsne_pcl_avg} are the average pooling of the layer-specific embeddings in Figure \ref{fig:visualization_views}.
Figure \ref{fig:tsne_no_reg} and \ref{fig:tsne_full} are X-GOAL w/o cluster-level alignment and the full X-GOAL.
Generally, the full X-GOAL best separates different clusters.

\section{Related Work}\label{sec:related}

\subsection{Contrastive Learning for Graphs}
The goal of CL is to pull similar nodes into close positions and push dis-similar nodes far apart in the embedding space.
Inspired by word2vec \cite{mikolov2013distributed}, early methods, such as DeepWalk \cite{perozzi2014deepwalk} and node2vec \cite{grover2016node2vec} use random walks to sample positive pairs of nodes.
LINE \cite{tang2015line} and SDNE \cite{wang2016structural} determine the positive node pairs by their first and second-order structural proximity.
Recent methods leverage graph transformation to generate node pairs.
DGI \cite{velivckovic2018deep}, GMI \cite{peng2020graph}, HDI \cite{jing2021hdmi} and CommDGI \cite{zhang2020commdgi} obtain negative samples by randomly shuffling the node attributes.
MVGRL \cite{hassani2020contrastive} transforms graphs via techniques such as graph diffusion \cite{klicpera2019diffusion}.
The objective of the above methods is to maximize the mutual information of the positive embedding pairs.
GraphCL \cite{you2020graph} uses various graph augmentations to obtain positive nodes.
GCA \cite{zhu2021graph} generates positive and negative pairs based on their importance.
gCool \cite{li2022graph} introduces graph communal contrastive learning.
Ariel \cite{feng2022adversarial, https://doi.org/10.48550/arxiv.2208.06956} proposes a information regularized adversarial graph contrastive learning.
These methods use the contrastive losses similar to InfoNCE \cite{oord2018representation}.

For multiplex heterogeneous graphs, MNE \cite{zhang2018scalable}, MVN2VEC \cite{shi2018mvn2vec} and GATNE \cite{cen2019representation} sample node pairs based on random walks.
DMGI \cite{park2020unsupervised} and HDMI \cite{jing2021hdmi} use random attribute shuffling to sample negative nodes.
HeCo \cite{wang2021self} decides positive and negative pairs based on the connectivity between nodes.
Above methods mainly rely on the topological structures to pair nodes, yet do not fully explore the semantic information, which could introduce semantic errors.

\subsection{Deep Clustering and Contrastive Learning}
Clustering algorithms \cite{xie2016unsupervised, caron2018deep} can capture the semantic clusters of instances.
DeepCluster \cite{caron2018deep} is one of the earliest works which use cluster assignments as ``pseudo-labels" to update the parameters of the encoder.
DEC \cite{xie2016unsupervised} learns a mapping from the data space to a lower-dimensional feature space in which it iteratively optimizes a clustering objective. 
Inspired by these works, SwAV \cite{caron2020unsupervised} and PCL \cite{li2020prototypical} combine deep clustering with CL. 
SwAV compares the cluster assignments rather than the embeddings of two images.
PCL is the closest to our work, which alternatively performs clustering to obtain the latent prototypes and train the encoder by contrasting positive and negative pairs of nodes and prototypes.
However, PCL has some limitations compared with the proposed X-GOAL: it is designed for single view image data; it heavily relies on data augmentations and momentum contrast \cite{he2020momentum}; it has some complex assumptions over cluster distributions and embeddings.

\subsection{Multiplex Heterogeneous Graph Neural Networks} 
The multiplex heterogeneous graph \cite{cen2019representation} considers multiple relations among nodes, and it is also known as multiplex graph \cite{park2020unsupervised, jing2021hdmi}, multi-view graph \cite{qu2017attention}, multi-layer graph \cite{li2018multi} and multi-dimension graph \cite{ma2018multi}. 
MVE \cite{qu2017attention} and HAN \cite{wang2019heterogeneous} uses attention mechanisms to combine embeddings from different views.
mGCN \cite{ma2019multi} models both within and across view interactions. 
VANE \cite{fu2020view} uses adversarial training to improve the comprehensiveness and robustness of the embeddings.
Multiplex graph neural networks have been used in many applications 
\cite{du2021new}, such as
time series \cite{jing2021network}, 
text summarization \cite{jing2021multiplex}, 
temporal graphs \cite{DBLP:conf/kdd/FuFMTH22},
graph alignment \cite{xiong2021contrastive}, 
abstract reasoning \cite{Wang2020Abstract}, global poverty \cite{khan2019multi} and bipartite graphs \cite{xue2021multiplex}.

\subsection{Deep Graph Clustering}
Graph clustering aims at discovering groups in graphs. 
SAE \cite{tian2014learning} and MGAE \cite{wang2017mgae} first train a GNN, and then run a clustering algorithm over node embeddings to obtain the clusters.
DAEGC \cite{wang2019attributed} and SDCN \cite{bo2020structural} jointly optimize clustering algorithms and the graph reconstruction loss.
AGC \cite{zhang2019attributed} adaptively finds the optimal order for graph filters based on the intrinsic clustering scores.
M3S \cite{sun2020multi} uses clustering to enlarge the labeled data with pseudo labels.  
SDCN \cite{bo2020structural} proposes a structural deep clustering network to integrate the structural information into deep clustering.
COIN \cite{jing2022coin} co-clusters two types of nodes in bipartite graphs.
MvAGC \cite{lin2021graph} extends AGC \cite{zhang2019attributed} to multi-view settings.
However, MvAGC is not neural network based methods which might not exploit the attribute and non-linearity information.
Recent methods combine CL with clustering to further improve the performance. 
SCAGC \cite{xia2021self} treats nodes within the same cluster as positive pairs.
MCGC \cite{pan2021multi} combines CL with MvAGC \cite{lin2021graph}, which treats each node with its neighbors as positive pairs. 
Different from SCAGC and MCGC, the proposed GOAL and X-GOAL capture the semantic information by treating a node with its corresponding cluster center as a positive pair.

\section{Conclusion}\label{sec:conclusion}
In this paper, we introduce a novel X-GOAL framework for multiplex heterogeneous graphs, which is comprised of a GOAL framework for each homogeneous graph layer and an alignment regularization to jointly model different layers.
The GOAL framework captures both node-level and cluster-level information.
The alignment regularization is a nimble technique to jointly model and propagate information across different layers, which could maximize the mutual information of different layers.
The experimental results on real-world multiplex heterogeneous graphs demonstrate the effectiveness of the proposed X-GOAL framework.






\appendix

\section{Derivation of Cluster-Level Loss}
The node-level contrastive loss is usually noisy, which could introduce semantic errors by treating two semantic similar nodes as a negative pair.
To tackle this issue, we use a clustering algorithm $\mathcal{C}$ (e.g. K-means) 
to obtain the semantic clusters of nodes, and we use the EM algorithm to update the parameters of $\mathcal{E}$ to pull node embeddings closer to their assigned clusters (or prototypes).

Following \cite{li2020prototypical}, we maximize the following log likelihood:
\begin{equation}\label{eq:obj}
    \sum_{n=1}^N\log p(\mathbf{h}_n|\mathbf{\Theta}, \mathbf{C}) = \sum_{n=1}^N\log \sum_{k=1}^Kp(\mathbf{h}_n, k|\mathbf{\Theta}, \mathbf{C})
\end{equation}
where $\mathbf{h}_n$ is the $n$-th row of $\mathbf{h}$, $\mathbf{\Theta}$ and $\mathbf{C}$ are the parameters of $\mathcal{E}$ and K-means algorithm $\mathcal{C}$, $k\in[1,\cdots, K]$ is the cluster index, and $K$ is the number of clusters.
Directly optimizing this objective is impracticable since the cluster index is a latent variable.

The Evidence Lower Bound (ELBO) of Equation \eqref{eq:obj} is given by:
\begin{equation}
    \text{ELBO} = \sum_{n=1}^N\sum_{k=1}^K Q(k|\mathbf{h}_n)\log \frac{p(\mathbf{h}_n, k|\mathbf{\Theta}, \mathbf{C})}{Q(k|\mathbf{h}_n)}
\end{equation}
where $Q(k|\mathbf{h}_n)=p(k|\mathbf{h}_n, \mathbf{\Theta}, \mathbf{C})$ is the auxiliary function.

In the E-step, we fix $\mathbf{\Theta}$ and estimate the cluster centers $\hat{\mathbf{C}}$ and the cluster assignments $\hat{Q}(k|\mathbf{h}_n)$ by running the K-means algorithm over the embeddings of the original graph $\mathbf{H}=\mathcal{E}(\mathcal{G})$.
If a node $\mathbf{h}_n$ belongs to the cluster $k$, then its auxiliary function is an indicator function satisfying 
$\hat{Q}(k|\mathbf{h}_n)=1$, and $\hat{Q}(k'|\mathbf{h}_n)=0$ for $\forall k'\neq k$.

In the M-step, based on $\hat{\mathbf{C}}$ and $\hat{Q}(k|\mathbf{h}_n)$ obtained in the E-step, we update $\mathbf{\Theta}$ by maximizing ELBO:
\begin{equation}
\begin{split}
    \text{ELBO} &= \sum_{n=1}^N\sum_{k=1}^K \hat{Q}(k|\mathbf{h}_n)\log p(\mathbf{h}_n, k|\mathbf{\Theta}, \hat{\mathbf{C}})\\
    &- \sum_{n=1}^N\sum_{k=1}^K\hat{Q}(k|\mathbf{h}_n)\log\hat{Q}(k|\mathbf{h}_n)
\end{split}
\end{equation}
Dropping the second term of the above equation, which is a constant, we will minimize the following loss function:
\begin{equation}\label{eq:semantic_loss_0}
\begin{split}
    \mathcal{L}_\mathcal{C} = -\sum_{n=1}^N\sum_{k=1}^K \hat{Q}(k|\mathbf{h}_n)\log p(\mathbf{h}_n, k|\mathbf{\Theta}, \hat{\mathbf{C}})
\end{split}
\end{equation}
Assuming a uniform prior distribution over $\mathbf{h}_n$, we have:
\begin{equation}
    p(\mathbf{h}_n, k|\mathbf{\Theta}, \hat{\mathbf{C}}) \propto p(k|\mathbf{h}_n, \mathbf{\Theta}, \hat{\mathbf{C}})
\end{equation}
We define $p(k|\mathbf{h}_n, \mathbf{\Theta}, \hat{\mathbf{C}})$ by:
\begin{equation}\label{eq:semantic_codes}
    p(k|\mathbf{h}_n, \mathbf{\Theta}, \hat{\mathbf{C}}) = \frac{\mathbf{e}^{(\hat{\mathbf{c}}_k^T\cdot\mathbf{h}_n/\tau)}}{\sum_{k'=1}^K\mathbf{e}^{(\hat{\mathbf{c}}_{k'}^T\cdot\mathbf{h}_n/\tau)}}
\end{equation}
where $\mathbf{h}_n\in\mathbb{R}^d$ is the embedding of the node $\mathbf{x}_n$, $\hat{\mathbf{c}}_k\in\mathbb{R}^d$ is the vector of the $k$-th cluster center, $\tau$ is the temperature parameter.

Let's use $k_n$ to denote the cluster assignment of $\mathbf{h}_n$, and normalize the loss by $\frac{1}{N}$, then Equation \eqref{eq:semantic_loss_0} can be rewritten as: 
\begin{equation}
    \mathcal{L}_\mathcal{C} = -\frac{1}{N}\sum_{n=1}^N\log \frac{\mathbf{e}^{(\mathbf{c}_{k_n}^T\cdot\mathbf{h}_n/\tau)}}{\sum_{k=1}^K\mathbf{e}^{(\mathbf{c}_{k}^T\cdot\mathbf{h}_n/\tau)}}
\end{equation}

The above loss function captures the semantic similarities between nodes by pulling nodes within the same cluster closer to their assigned cluster center.

\begin{acks}
BJ and HT are partially supported by NSF (1947135, 
2134079 
and 1939725
), and NIFA (2020-67021-32799).
\end{acks}

\bibliographystyle{ACM-Reference-Format}
\balance
\bibliography{sample-base}

\end{document}